%% file: switching.tex
\documentclass[12pt]{article}
\usepackage{fullpage}

\usepackage{amsmath}
\usepackage{amsthm}
\usepackage{amsfonts}
\usepackage{amssymb}
\usepackage{natbib}
\usepackage{tikz}
\usepackage{algorithmic}
\usepackage{float}

\setcitestyle{authoryear,round,citesep={;},aysep={,},yysep={;}}

\floatstyle{boxed}
\newfloat{algorithm}{t}{lop}

\newtheorem{theorem}{Theorem}
\newtheorem{lemma}{Lemma}
\newtheorem{corollary}{Corollary}
\newtheorem{definition}{Definition}

\input{macro}

\begin{document}

\title{Bandits with Switching Costs: $T^{2/3}$ Regret}

\author{%
\makebox[0.4\linewidth]{Ofer Dekel}\\
Microsoft Research\\
\texttt{oferd@microsoft.com}
\and
\makebox[0.4\linewidth]{Jian Ding\thanks{Most of this work was done while the author was at Microsoft Research, Redmond.}}\\
University of Chicago\\
\texttt{jianding@galton.uchicago.edu}
\and\\
\makebox[0.4\linewidth]{Tomer Koren\footnotemark[1]}\\
Technion\\
\texttt{tomerk@technion.ac.il}
\and\\
\makebox[0.4\linewidth]{Yuval Peres}\\
Microsoft Research\\
\texttt{peres@microsoft.com}
}

\date{}

\maketitle

\begin{abstract}%
We study the adversarial multi-armed bandit problem in a setting where
the player incurs a unit cost each time he switches actions.  We prove
that the player's $T$-round minimax regret in this setting is
$\widetilde{\Theta}(T^{2/3})$, thereby closing a fundamental gap in
our understanding of learning with bandit feedback. In the
corresponding full-information version of the problem, the minimax
regret is known to grow at a much slower rate of $\Theta(\sqrt{T})$.
The difference between these two rates provides the \emph{first}
indication that learning with bandit feedback can be significantly
harder than learning with full-information feedback (previous results
only showed a different dependence on the number of actions, but not
on $T$.) 

In addition to characterizing the inherent difficulty of the
multi-armed bandit problem with switching costs, our results also
resolve several other open problems in online learning. One direct
implication is that learning with bandit feedback against
bounded-memory adaptive adversaries has a minimax regret of
$\widetilde{\Theta}(T^{2/3})$. Another implication is that the minimax
regret of online learning in adversarial Markov decision processes
(MDPs) is $\widetilde{\Theta}(T^{2/3})$. The key to all of our results
is a new randomized construction of a multi-scale random walk,
which is of independent interest and likely to prove useful in additional 
settings.
\end{abstract}

\newpage

\section{Introduction}
Online learning with a finite set of actions is a fundamental problem
in machine learning, with two important special cases: the
\emph{Adversarial (Non-Stochastic) Multi-Armed Bandit} \citep{Auer:02}
and \emph{Predicting with Expert Advice}
\citep{CesaFrHaHeScWa97,FS97}. This problem is often presented as a
$T$-round repeated game between a player and an adversary: on each
round of the game, the player chooses an \emph{action}\footnote{In the
  bandit problem, each action is called an \emph{arm}; in the experts
  problem, each action is called an \emph{expert}.} from the set $[k]
= \{1,\ldots,k\}$ and incurs a loss in $[0,1]$ for that action. The
player is allowed to randomize, i.e., on each round he selects a
distribution over actions and draws an action from that
distribution. The loss corresponding to each action on each round is
set in advance by the adversary, and in particular, the loss of each
action can vary from round to round. The player's goal is to minimize
the total loss accumulated over the course of the game.

The bandit problem and the experts problem differ in the feedback
received by the player after each round. In the bandit problem, the
player only observes his loss (a single number) on each round; this is
called \emph{bandit feedback}. In the experts problem, the player
observes the loss assigned to each possible action (for a total of $k$
real numbers in each round); this is called \emph{full feedback} or
\emph{full information}. A player that receives bandit feedback must
balance an exploration/exploitation trade-off, while a player that
receives full feedback is only concerned with exploitation.

For example, say that we manage an investment portfolio, we receive
daily advice from $k$ financial experts, and on each day we must
follow the advice of one expert. The loss associated with each expert
on each day reflects the amount of money we would lose by following
that expert's advice on that day. If we know the advice given by each
expert, the problem is said to provide full feedback. Alternatively,
if we purchase advice from a single expert on each day, and the
advice of the other $k-1$ experts remains unknown, the problem is said
to provide bandit feedback.



In the problem just described, the player is allowed to switch freely
between actions. An equally interesting setting is one where each
switch incurs a \emph{switching cost}: In addition to the losses
chosen by the adversary, the player pays a penalty each time his
action differs from the one he played on the previous round. In the
motivating example described above, switching our primary financial
consultant may require terminating a contract with the previous expert
and negotiating contract with the new one, or it may just cost us the
fees and commissions that result from a significant change in
investment strategy. Switching costs arise naturally in a variety of
other applications: In online web applications, switching the content
of a website too frequently can be annoying to users; in industrial
applications, switching actions might entail reconfiguring a
production line.  Moreover, \citet{Geulen:10} reduced a family of
online buffering problems to switching cost problems; similarly,
\citet{gyorgy2011near} used the switching cost setting to solve the
limited-delay universal lossy source coding problem.

We focus on analyzing the inherent difficulty of online learning with
switching costs, using the game-theoretic notion of \emph{minimax
  regret}. To define this notion, we must first specify the setting
formally. Before the game begins, the adversary chooses a loss
functions $\ell_1,\ldots,\ell_T$, where each $\ell_t$ maps the action
set $[k]$ to $[0,1]$. Since the entire sequence is chosen in advance,
we say that the adversary is \emph{oblivious} (to the player's
actions). On round $t$, the player selects a distribution over the set
of actions and draws an action $X_t$ from that distribution. The
player then incurs the loss $\ell_t(X_t) + 1\!\!1_{X_t \neq X_{t-1}}$,
which includes the adversarially chosen loss $\ell_t(X_t)$ and the
switching cost. To make the loss on the first round well-defined, we set $X_0=0$ (so the
first action always counts as a switch). The player's cumulative loss
at the end of the game equals $\sum_{t=1}^T \big(\ell_t(X_t) +
1\!\!1_{X_t \neq X_{t-1}}\big)$.

Since the loss functions are adversarial, the cumulative loss is only
meaningful when compared to an adequate baseline. Therefore, we
compare the player's cumulative loss to the loss of the best fixed
policy (in hindsight), which is a policy that chooses the same action
on all $T$ rounds. Formally, we define the player's \emph{regret} at
the end of the game as
\begin{equation}\label{eqn:minimax}
R ~=~ \sum_{t=1}^T \left( \ell_t(X_t) + \ind{X_t \neq X_{t-1}} \right) ~-~ \min_{x \in [k]} \sum_{t=1}^T \ell_t(x) ~~.
\end{equation}
While regret measures the player's performance on a given instance of
the game, the inherent difficulty of the game itself is measured by
\emph{minimax expected regret} (or just \emph{minimax regret} for
brevity). Intuitively, minimax regret is the expected regret when both
the adversary and the player behave optimally. Formally, minimax
regret is the minimum over all randomized player strategies, of the
maximum over all loss sequences, of $\E[R]$. In this paper, our
primary focus is to determine the asymptotic growth rate of the
minimax regret as a function of the number of rounds $T$ and the
number of actions $k$.

Minimax regret rates are already well understood in several of the
settings discussed above. Without switching costs, the minimax regret
of the adversarial multi-armed bandit problem is $\Theta(\sqrt{Tk})$
(see \citet{Auer:02,CesaBianchi:06}) and the minimax regret of the
experts problem is $\Theta(\sqrt{T \log{k}})$ (see
\citet{LW94,FS97,CesaBianchi:06}). This implies that when no switching
costs are added, the bandit problem is not substantially more difficult 
than the experts problem (at least when the number
of actions is constant), despite the added burden of exploration.

When switching costs are added, the previous literature does not
provide a full characterization of minimax regret.  Clearly, the lower
bound without switching costs still apply with switching costs are
added. In the full feedback setting with switching costs, the
\emph{Follow the Lazy Leader} algorithm \citep{Kalai:05} and the
\emph{Shrinking Dartboard} algorithm \citep{Geulen:10} both guarantee
a matching upper bound of $O(\sqrt{T \log{k}})$, so the minimax regret
is $\Theta(\sqrt{T \log{k}})$. However, the minimax regret of the
bandit problem with switching costs was not well
understood. \citet{Arora:12} presented a simple algorithm with a guaranteed
regret of $O(k^{1/3} T^{2/3})$, but a matching lower bound was not
known.

Recently, \citet{CesaBianchiDeSh13} addressed this gap, but fell short
of resolving it. Specifically, they modified the game by allowing the
loss per round to drift out of the interval $[0,1]$ and to possibly
grow in magnitude to be as large as $\Theta(\sqrt{T})$.  
In this setting, they proved that the
minimax regret (with a constant number of actions $k$) grows at a rate
of $\widetilde\Theta(T^{2/3})$.  However, allowing unbounded loss per
round is quite uncommon and not very natural. Also, it isn't clear
what implications their results have on the original problem (i.e., with
bounded losses), and whether their $\widetilde\Theta(T^{2/3})$ rate is 
merely an artifact of the enlarged range of admissible loss values.

\subsection{Our Results}

Our main result is a new $\widetilde \Omega(T^{2/3})$ lower bound on the
regret of the multi-armed bandit problem with switching costs (in the standard setup, with
losses bounded in $[0,1]$). 

\begin{theorem} \label{thm:main}
For any randomized player strategy that relies on bandit feedback,
there exists a sequence of loss functions $\ell_1,\ldots,\ell_T$ (where
$\ell_t:[k] \mapsto [0,1]$) that incurs a regret
of $R = \widetilde \Omega(k^{1/3} T^{2/3})$,
provided that $k \le T$.
\end{theorem}

When combined with the upper bound in \citet{Arora:12}, our result
implies that the minimax regret of the multi-armed bandit problem with
switching costs is $\widetilde \Theta(k^{1/3} T^{2/3})$. Thus when
switching costs are added, the bandit problem becomes substantially
more difficult than the corresponding experts problem. To the best of our
knowledge, this is the first example that exhibits (even for constant $k$) a clear gap
between the asymptotic difficulty, as $T$ grows, of online learning with bandit and
full feedback.

To prove \thmref{thm:main}, we apply (the easy direction of) Yao's
minimax principle \citep{Yao77}, which states that the regret of a
randomized player against the worst-case loss sequence is at least
 the minimax regret of the optimal \emph{deterministic}
player against a \emph{stochastic} loss sequence. In other words, as
an intermediate step toward proving \thmref{thm:main}, we construct a
stochastic sequence of loss functions\footnote{We use the notation $U_{i:j}$ as shorthand for the sequence $U_i,\ldots,U_j$ throughout.}, $L_{1:T}$, where each $L_t$ is
a random function from $[k]$ to $[0,1]$, such that
$$
\E\left[ \sum_{t=1}^T \big( L_t(X_t) + 1\!\!1_{X_t \neq X_{t-1}} \big) ~-~ \min_{x \in [k]} \sum_{t=1}^T L_t(x) \right]
~=~ \widetilde \Omega(k^{1/3} T^{2/3}) ~,
$$
for any deterministic player strategy.
%
%

After proving our lower bound for constant switching costs, we
generalize is to arbitrary switching costs (e.g., set the switching
cost to $T^q$, for some $q \in [-1,1]$). Additionally, we prove that
any algorithm for the multi-armed bandit problem that guarantees a
regret of $O(\sqrt{T})$ (without switching costs), such as the
algorithm presented in \citet{Auer:02}, can be forced to make
$\t\Omega(T)$ switches.
Finally, we observe that our problem is a
special case of an online Markov decision process (MDP) learning
problem with adversarial rewards and bandit feedback, and therefore
the minimax regret of that problem is also $\t\Omega(T^{2/3})$.

The paper is organized as follows: in \secref{sec:construction} we
describe the general construction of the stochastic loss sequence and
in \secref{sec:stochasticProcess} we present the stochastic process
that underlies our construction. We then prove our lower bound on
regret in \secref{sec:analysis} and present extensions and
implications in \secref{sec:extensions}.

\section{Constructing the Loss Sequence} \label{sec:construction}

\begin{figure}[t]
\renewcommand{\algorithmicrequire}{\textbf{Input:}}
\renewcommand{\algorithmicensure}{\textbf{Output:}}
\centering
\fbox{
\begin{minipage}{0.75\linewidth}
\begin{algorithmic}[1]
        \vspace{0.1cm}
	\REQUIRE time horizon $T > 0$, number of actions $k \ge 2$\\[0.1cm]
	
	\STATE Set $\eps = k^{1/3} T^{-1/3} / (9\log_2{T})$ and $\sig = 1/(9\log_2{T})$.\\[0.1cm]

	\STATE Choose $\chi \in [k]$ uniformly at random.\\[0.1cm]
	\STATE Draw $T$ independent zero-mean $\sig^2$-variance Gaussians $\xi_{1:T}$.\\[0.1cm]

	\STATE Define $W_{0:T}$ recursively by
	\begin{align*}
		W_0 &= 0~~, \\
		\forall~t \in [T]~~~~ W_t &= W_{\rho(t)} + \xi_t~~,
	\end{align*}
	where
		$\rho(t) = t - 2^{\delta(t)} \,,~
		\delta(t) = \max\set{i \ge 0 : 2^i \text{ divides } t} $.
	\smallskip

	\STATE For all $t \in [T]$ and $x \in [k]$, set
	\begin{align*}
	L'_t(x) &= W_t+\tfrac{1}{2} - \eps \cdot \ind{\chi = x} ,\\
	L_t(x) &= \clip \big(L'_t(x)\big),
	\end{align*}
	where $\clip(\alpha) = \min\{\max\{ \alpha, 0 \}, 1\}$.\\[0.1cm]
	\ENSURE loss functions $L_{1:T}$.\\[-0.2cm]~
\end{algorithmic}
\end{minipage}
}
\caption{The adversary's randomized algorithm for generating a loss
  sequence $L_{1:T}$, which ensures an expected regret of
  $\widetilde\Omega(k^{1/3} T^{2/3})$ against any deterministic player.}
\label{fig:adversary}
\end{figure}

\noindent
In this section we present our construction of a stochastic sequence of loss functions, $L_{1:T}$, which ensures an expected regret of $\widetilde\Omega(k^{1/3} T^{2/3})$ against any deterministic player.
The adversary's algorithm for generating the sequence $L_{1:T}$ is
given in \figref{fig:adversary}. The key to this algorithm is the
stochastic process $W_{1:T}$, defined on lines 3--4 of
\figref{fig:adversary}. The adversary draws a concrete sequence from
this process and uses it to define the loss values of all $k$ actions.
First, the adversary picks an action~$\chi \in [k]$ uniformly at
random to serve as the best action (whose loss is always smaller than
the loss of the other actions), and defines the intermediate loss
function sequence $L'_{1:T}$, whose values are not guaranteed to be
bounded in $[0,1]$. The loss of all actions $x \ne \chi$ is simply set
to $L'_t(x) = W_t + \half$. The loss of the best action $\chi$ is set
to $L'_t(\chi) = W_t + \half - \eps$, where $\eps$ is a predefined gap
parameter, and is therefore consistently better than the losses of the
other actions.  The loss sequence $L_{1:T}$ is obtained by taking the
intermediate sequence $L'_{1:T}$ and projecting each of its values to
the interval $[0,1]$.

When faced with the loss sequence $L_{1:T}$, the player attempts to
identify which of the $k$ actions has the smaller loss (or
equivalently, to reveal the value of $\chi$). Although the loss values
of the best action are deterministically separated from those of the
other actions by a constant gap, the player only observes one loss
value on each round, and never knows if his chosen action incurred the
higher loss or the lower loss.  Our analysis shows that the player's
ability to uncover information about the identity of the best action
depends on the characteristics of the stochastic process
$W_{1:T}$. For example, if this process were an i.i.d.~sequence, it is
easy to see that the player could identify the best action by
estimating the expected loss of every action to within $\eps/2$ (for
example, using Hoeffding's bound), requiring only $O(\sigma^2/\eps^2)$
samples of each action and at most $k-1$ switches between actions.
This example already implies that the dependency structure in our
construction of $W_{1:T}$ plays a central role.  We show that a
careful choice of the stochastic process $W_{1:T}$ ensures that the
amount of information uncovered by the player during the game is
tightly controlled by the number of switches he performs. Therefore,
to detect the best action, the player must switch actions frequently
and pay the associated switching costs.


\section{The Stochastic Process}\label{sec:stochasticProcess}

The key to our analysis is a careful choice of the stochastic process
$W_{1:T}$ that underlies the definition of $L_{1:T}$.  In this section
we describe a stochastic processes with a controllable dependence
structure, which includes i.i.d. Gaussian sequences and simple
Gaussian random walks as special cases.

Let $\xi_{1:T}$ be a sequence of independent zero-mean Gaussian random
variables with variance $\sigma^2$. Let $\rho:[T] \mapsto \{0\} \cup
[T]$ be a function that assigns each $t \in [T]$ with a \emph{parent}
$\rho(t)$. We allow $\rho$ to be any function that satisfies $\rho(t)
< t$ for all $t$.  Now define
\begin{align*}
W_0 &= 0~~, \\
\forall~t \in [T]~~~~ W_t &= W_{\rho(t)} + \xi_t~~.
\end{align*}
Note that the constraint $\rho(t) < t$ guarantees that a recursive
application of $\rho$ always leads back to zero.
%
The definition of the parent function $\rho$ determines the behavior
of the stochastic processes. For example, setting $\rho(t) = 0$
implies that $W_t = \xi_t$ for all $t$, so the stochastic process is
simply a sequence of i.i.d.~Gaussians. On the other hand, setting
$\rho(t) = t-1$ results in a simple Gaussian random walk.
Other definitions of $\rho$ can create interesting dependencies
between the variables of the stochastic process.

\subsection{Depth and Width}

We highlight two properties of the parent function $\rho$ (and consequently, of the induced stochastic process) that are essential to our analysis.

\begin{definition}[\emph{ancestors, depth}]
Given a parent function $\rho$, the set of ancestors of
$t$ is denoted by~\,$\anc(t)$ and defined as the set of positive indices
that are encountered when $\rho$ is applied recursively to
$t$. Formally, $\anc(t)$ is defined recursively as
\begin{align}
  &\anc(0) = \{\} \nonumber \\
\forall~t \in [T]~~~~& \anc(t) =  \anc\big(\rho(t)\big)~\cup~\{\rho(t)\}~~. \label{eqn:ancestor}
\end{align}
The depth of $\rho$ is then defined as $\depth(\rho) = \max_{t \in [T]} |\anc(t)|$.
\end{definition}

Using this definition, we can write $W_t = \xi_t+ \sum_{s \in \anc(t)} \xi_s$, where $\xi_0=0$.
Thus, if $\depth(\rho)=d$, the induced stochastic process includes
sums of at most $d$ independent Gaussians, each with variance $\sigma^2$.
This implies the following bound.

\begin{lemma} \label{lem:depth}
Let $W_{1:T}$ be the stochastic process defined by the parent function $\rho$.
Then 
\begin{align*}
\forall \delta\in(0,1) \qquad	\Pr\lr{\max_{t \in [T]} \abs{W_t}
	\le \sig \sqrt{2\depth(\rho) \log\tfrac{T}{\delta}} }  \ge 1-\delta~.
\end{align*}
\end{lemma}

\begin{proof}
For any $t\in[T]$, $W_t$ is 
normally distributed with zero mean and variance bounded by $\depth(\rho)
\sig^2$.  Since a standard Gaussian variable $Z$ satisfies $\Pr(\abs{Z} \ge z) \le \exp(-\frac{1}{2} z^2)$ for any $z \ge 0$,
we infer that
\begin{align*}
	\Pr\lr{\abs{W_t} \ge \sig \sqrt{2 \depth(\rho) \log\tfrac{T}{\delta}}}
	\le \exp\lr{- \log\tfrac{T}{\delta}}
	= \frac{\delta}{T}~~.
\end{align*}
The above holds for each $t \in [T]$ and the lemma follows from the union bound.
\end{proof}

\lemref{lem:depth} implies that the depth of $\rho$ and the variance
$\sigma^2$ determine how far the process $W_{1:T}$ will drift.  Since
we require a process that is bounded with high probability, we need to
minimize the depth of $\rho$.  (We could counter the effect of a deep
$\rho$ by setting $\sigma$ to be small, but if we do so, the resulting
process would not be able to mask the $\epsilon$ gap between the
losses of the different actions.) This consideration rules out the
simple Gaussian random walk, whose depth is $T$.

\begin{definition}[\emph{cut, width}]
Given a parent function $\rho$, define  
$$
\cut(t) ~=~ \set{s \in [T] : \rho(s) < t \le s}~,
$$
the set of rounds that are separated from their parent by $t$.
The width of $\rho$ is then defined\,\footnote{The width of $\rho$ coincides with the cut-width of the numbered graph it determines, see~\cite{CuSe89}.}  as $\width(\rho) = \max_{t \in [T]}|\cut(t)|$.
\end{definition}

Note that the cut size for any $s \in [T]$ is an integer between $1$
and $T$. One extreme is the simple Gaussian random walk ($\rho(t) = t-1$), whose cuts
are of size $1$. The other extreme is the sequence of
i.i.d. Gaussians ($\rho(t)=0$), for which
$|\cut(s)| = s$, and therefore $\width(\rho)=T$.








Our analysis in \secref{sec:infoBound} shows that any  
information that the player uncovers about the identity of the best
action can be attributed to a switch performed on the current round or
on a past round (where the first round is always considered to be a
switch). Moreover, we prove that the amount of information that can be
extracted from a switch at time $t$ is controlled by the size of
$\cut(t)$. Therefore, a process with a small width forces the player
to perform many switches.  This rules out the sequence of
i.i.d.~Gaussians, as it is too wide and reveals too much information
to a player that selects the same action repeatedly.

\subsection{The Multi-scale Random Walk} \label{sec:mmrw}

To prove our lower bound, we require a stochastic process that is
neither too deep nor too wide. We present such a process, called the
\emph{Multi-scale Random Walk} (MRW), whose depth and width
are both \emph{logarithmic} in $T$.
The MRW process is formed by the parent function given by
\begin{equation} \label{eq:rho}
	\rho(t) = t - 2^{\delta(t)} \,,
	\quad\text{where}\quad
	\delta(t) = \max\set{i \ge 0 : 2^i \text{ divides } t} \,.
\end{equation}
Put another way, $\rho(t)$ is obtained by taking the binary
representation of $t$, identifying the lowest order $1$, and flipping
it to $0$.  For example if $t=10110{\bf 1}00$ (which equals the
decimal number $180$) then $\rho(t) = 10110{\bf 0}00$ (which equals the
decimal number $176$).

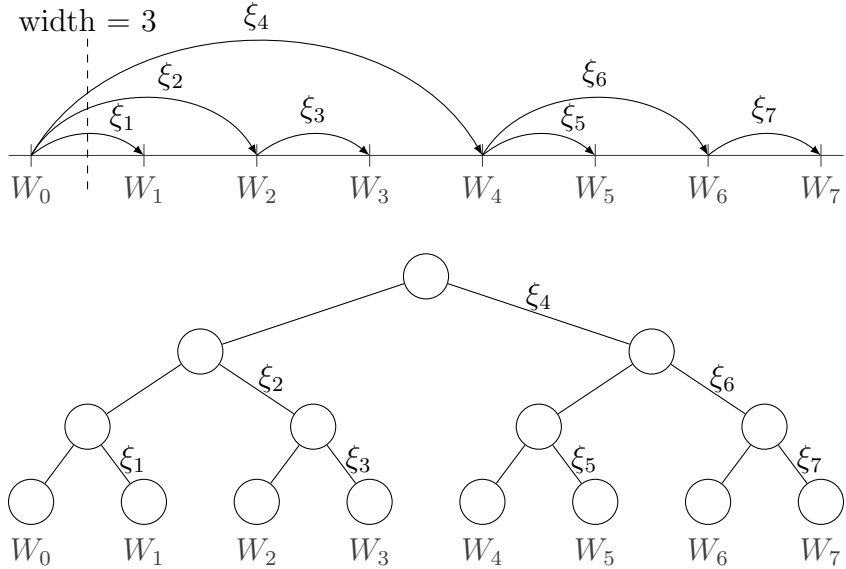
\begin{figure}[t]
\begin{center}
\begin{tikzpicture}[scale=1.5]
\draw[darkgray] (-0.2,0) -- (7.2,0);
\foreach \x in {0,...,7}
   \draw[darkgray] (\x,-.1) -- (\x,.1);

\path[-latex,black] (0,0) edge[bend left=60] node[above]{$\xi_{4}$} (4,0);

\path[-latex,black] (0,0) edge[bend left=60] node[above,pos=0.6]{$\xi_{2}$} (2,0);
\path[-latex,black] (4,0) edge[bend left=60] node[above]{$\xi_{6}$} (6,0);

\path[-latex,black] (0,0) edge[bend left=40] node[above,pos=0.8]{$\xi_{1}$} (1,0);
\path[-latex,black] (2,0) edge[bend left=40] node[above]{$\xi_{3}$} (3,0);
\path[-latex,black] (4,0) edge[bend left=40] node[above,pos=0.8]{$\xi_{5}$} (5,0);
\path[-latex,black] (6,0) edge[bend left=40] node[above]{$\xi_{7}$} (7,0);

\node[darkgray] at (0,-.3) {$W_0$};
\node[darkgray] at (1,-.3) {$W_1$};
\node[darkgray] at (2,-.3) {$W_2$};
\node[darkgray] at (3,-.3) {$W_3$};
\node[darkgray] at (4,-.3) {$W_4$};
\node[darkgray] at (5,-.3) {$W_5$};
\node[darkgray] at (6,-.3) {$W_6$};
\node[darkgray] at (7,-.3) {$W_7$};

\draw[dashed] (0.5,-0.3) -- (0.5,1.1);
\node at (0.5,1.2){width = 3};

\end{tikzpicture}
\vskip 0.5cm
\tikzstyle{level 1}=[level distance=1cm, sibling distance=6cm]
\tikzstyle{level 2}=[level distance=1cm, sibling distance=3cm]
\tikzstyle{level 3}=[level distance=1cm, sibling distance=1.5cm]
\begin{tikzpicture}[grow=down,inner sep=0pt, minimum size=0.6cm]
\node[circle,draw]{}
    child{
        node[circle,draw] {}
            child {
                node[circle,draw] {}
                     child {
                       node[circle,draw](A) {}
                     }
                     child {
                       node[circle,draw](B) {}
                      edge from parent node[above=2pt,pos=1.1] {$\xi_1$};
                     }
            }
            child {
                node[circle,draw] {}
                     child {
                       node[circle,draw](C) {}
                     }
                     child {
                       node[circle,draw](D) {}
                      edge from parent node[above=2pt,pos=1.1] {$\xi_3$};
                     }
            edge from parent node[above=-1pt,pos=0.7] {$\xi_2$};
           }
    }
    child {
        node[circle,draw] {}
            child {
                node[circle,draw] {}
                     child {
                       node[circle,draw](E) {}
                     }
                     child {
                       node[circle,draw](F) {}
                      edge from parent node[above=2pt,pos=1.1] {$\xi_5$};
                     }
            }
            child {
                node[circle,draw] {}
                     child {
                       node[circle,draw](G) {}
                     }
                     child {
                       node[circle,draw](H) {}
                      edge from parent node[above=2pt,pos=1.1] {$\xi_7$};
                     }
            edge from parent node[above=-1pt,pos=0.7] {$\xi_6$};
            }
        edge from parent node[above=-2pt] {$\xi_4$};
    };
\node[darkgray,below=10pt] at (A) {$W_0$};
\node[darkgray,below=10pt] at (B) {$W_1$};
\node[darkgray,below=10pt] at (C) {$W_2$};
\node[darkgray,below=10pt] at (D) {$W_3$};
\node[darkgray,below=10pt] at (E) {$W_4$};
\node[darkgray,below=10pt] at (F) {$W_5$};
\node[darkgray,below=10pt] at (G) {$W_6$};
\node[darkgray,below=10pt] at (H) {$W_7$};
\end{tikzpicture}
\end{center}

\caption{An illustration of the MRW process for $T=7$. (Top) The MRW
  with a directed edge from $\rho(t)$ to $t$, for each
  $t\in[T]$. (Bottom) The MRW can be equivalently described as the
  values at the leaves of a binary tree, where the value at each
  leaf is obtained by summing the i.i.d.\ Gaussian variables $\xi_t$'s on the (right) edges
  along the path from the root. }
\label{fig:mrw}
\end{figure}

\figref{fig:mrw} depicts the MRW process for $T=7$.  Notice that the
process takes steps on \emph{multiple scales}, each of which
corresponds to a different power of two.  An alternative description
of the same process can be obtained by considering a binary tree with
leaves corresponding to the random variables $W_{1:T}$, as depicted in
\figref{fig:mrw}.  In this description, we associate the \emph{right}
edges of the tree, enumerated in a DFS traversal order, with the
Gaussian variables $\xi_{1:T}$.  Then, each $W_t$ is defined as the
sum of the $\xi_j$'s encountered along the path from the root to the
leaf corresponding to $W_t$.

We conclude the section with the following lemma, which summarizes the
properties of the MRW process used in our analysis.

\begin{lemma} \label{lem:mrw}
The depth and width of the MRW are both upper-bounded by $\floor{\log_2{T}}+1$.
\end{lemma}

\begin{proof}
Let $n = \floor{\log_2{T}}+1$ and note that any integer $t \in [T]$
can be written using $n$ bits.  We shall prove that, for all $t \in
[T]$, the number $\abs{\anc(t)}$ is bounded by (in fact, is equal to) the
number of 1's in the $n$-digit binary representation of $t$, while
$\abs{\cut(t)}$ is bounded by the number of 0's in that representation
plus one.  This would immediately imply the lemma, as $\abs{\anc(t)}$ and
$\abs{\cut(t)}$ are both positive and their sum is at most $n+1$.

First, observe that the number of 1's in the representation of the
parent $\rho(t)$ is one less than the number of 1's in the
representation of $t$, and $\abs{\anc(0)} = 0$.  Hence, $\abs{\anc(t)}$
equals the number of 1's in the binary representation of $t$.

Moving on to the width, choose any $t \in [T]$ and consider the cut it defines.
We show that each $s \in \cut(t) \setminus \set{t}$ corresponds to a distinct zero in the $n$-bit binary representation of $t$.
Let~$s \in \cut(t) \setminus \set{t}$ and denote $j = \delta(s)$.
Note that $\rho(s) = s - 2^j$ is a multiple of $2^{j+1}$, so we can write $s - 2^j = a \cdot 2^{j+1}$ for some integer $a$.
By the definition of the cut and since $s \ne t$, we have $a \cdot 2^{j+1} < t < a \cdot 2^{j+1} + 2^j$.
Consequently,   $s = 2^{j+1} \cdot \floor{t/2^{j+1}} + 2^j$ and the coefficient of $2^j$ in the binary representation of $t$ is zero.
Together with the fact that $t \in \cut(t)$, we have shown that the size of the cut defined by $t$ is at most the number of zero bits in its binary representation plus one.
\end{proof}

\section{Analysis} \label{sec:analysis}

In this section, we prove our main result: a $\t\Omega(k^{1/3} T^{2/3})$ lower
bound on the expected regret of the multi-armed bandit with switching
costs, when the loss functions are stochastic and the player is
deterministic. Our result is stated formally in the following theorem.

\begin{theorem} \label{thm:regret-lb}
Let $L_{1:T}$ be the stochastic sequence of loss functions defined in
\figref{fig:adversary}.
Then for  $T \ge \max\set{k,6}$, the expected regret (as defined in
\eqref{eqn:minimax}) of any deterministic player against this sequence
is at least $k^{1/3} T^{2/3}/ (100 \log_2{T})$.
\end{theorem}

Our analysis requires some new notation. First, let $M = \sum_{t=1}^T
\ind{X_t \neq X_{t-1}}$ be the number of switches in the action
sequence $X_{1:T}$ (recall that we arbitrarily set $X_0=1$). Also, for
all $t\in[T]$, let $Z_t = L_t(X_t)$ be the loss observed by the player
on round $t$. Recall our assumption that $X_t$, the player's action on round
$t$,  is a deterministic function of his past observations
$Z_{1:(t-1)}$.

\subsection{Distinguishability Requires Switching}
\label{sec:infoBound}

We begin the analysis with a key lemma that relates the player's
ability to identify the best action to the number of switches he
performs. This lemma also highlights the importance of finding a
stochastic process with a small $\width(\rho)$.  
The lemma bounds the distance between each one of the conditional probability measures
$$
\Q_i(\cdot) = \Pr(\cdot \,|\, \chi=i) ~, \qquad i = 1,2,\ldots,k~,
$$
and the probability measure $\Q_0$ that corresponds to an (imaginary) adversary that uses $\chi = 0$.
%
Thus $\Q_0(\cdot)$ is the probability when all actions incur the same loss.
Let $\Fcal$ be the
$\sigma$-algebra  generated
by the player's observations $Z_{1:T}$. Then the \emph{total
  variation} distance between $\Q_0$ and $\Q_i$ on
$\Fcal$ is defined as
$$
\tv{\Q_0}{\Q_i} ~=~
\sup_{A \in \Fcal} \big|\Q_0(A) - \Q_i(A) \big| ~~.
$$
This distance captures the player's ability to identify whether action $i$ is better than or equivalent to the other actions based on the loss values he observes.
The following lemma upper-bounds this distance in terms of the number of switches the player performs to or from action $i$, denoted by the random variable $M_i$, and the width $\width(\rho)$ of the underlying stochastic process. 
Here we use the notation $\E_{\Q_j}$ to refer to the expectation with respect to the distribution $\Q_j$, for any $j=0,1,\ldots,k$.

\begin{lemma} \label{lem:tv-bound} For all $i \in [k]$, it holds that
$d_\mathrm{TV}^{\Fcal}(\Q_0,\Q_i) \le (\eps/2\sig) \sqrt{\width(\rho) \, \E_{\Q_0}[M_i]}$ and 
$d_\mathrm{TV}^{\Fcal}(\Q_0,\Q_i) \le (\eps/2\sig) \sqrt{\width(\rho) \, \E_{\Q_i}[M_i]}$.
\end{lemma}
To see the significance of this lemma, consider first the case $k=2$, where $M_1=M_2=M$ by definition.
By the triangle inequality, $\tv{\Q_1}{\Q_2} \le \tv{\Q_0}{\Q_1} + \tv{\Q_0}{\Q_2}$.
Concavity of square root yields 
$$
	\sqrt{\E_{\Q_1}[M]} + \sqrt{\E_{\Q_2}[M]} 
	\le \sqrt{2\,(\E_{\Q_1}[M]+ \E_{\Q_2}[M])} 
	= 2 \sqrt{\E[M]}~.
$$
The second claim of \lemref{lem:tv-bound} for $k=2$ now implies that $\tv{\Q_1}{\Q_2} \le (\eps/\sig) \sqrt{\width(\rho) \, \E[M]}$.
This inequality clarifies the dilemma facing the player: If he switches actions frequently so that  $\E[M]=\Omega(T^{2/3}/\log(T))$,
the switching costs guarantee the desired lower bound on regret. Otherwise,  $\E[M]=o(T^{2/3}/\log(T))$ ; since $\eps/\sig=\Theta(T^{-1/3})$ and 
$\width(\rho)=\Theta(\log(T))$, the distance
$d_\mathrm{TV}^{\Fcal}(\Q_1,\Q_2)$ will tend to zero with $T$, so the player will be unable to distinguish between the two actions and will suffer an expected regret of order 
$\Theta(\eps T)=\Theta(T^{2/3}/\log(T))$. We do not formalize this argument here, since we prove the lower bound for any $k$ below.


\begin{proof}[Proof of Lemma~\ref{lem:tv-bound}]
Let $Y_0=\half$ and $Y_t = L'(X_t)$ for all $t\in[T]$. Note that $X_t$ is a
deterministic function of $Y_{0:(t-1)}$. Define $Y_S = \{Y_t\}_{t \in
  S}$ and let $\D(Y_S \mid Y_{S'})$ be the relative entropy
(i.e., the Kullback-Leibler divergence) between the joint distribution of $Y_S$,
conditioned on $Y_{S'}$, under $\Q_0$ and $\Q_i$. Namely,
\begin{align}
\D(Y_S \mid Y_{S'})
&~=~ \E_{\Q_0}\left[\log \frac{\Q_0(Y_S \mid Y_{S'})}{\Q_i(Y_S \mid Y_{S'})} \right]~.
\end{align}
For brevity, also define $\D(Y_S) = \D(Y_S \mid \emptyset)$.
We use the chain rule for relative entropy (see, e.g., Theorem 2.5.3 in \cite{cover2006elements}) to decompose $\D(Y_{0:T})$ as
\begin{equation} \label{eq:klchain}
\D(Y_{0:T}) ~=~ \D(Y_0) + \sum_{t=1}^T \D\big(Y_t \mid Y_{\anc(t)}\big)
\end{equation}
and deal separately with each term in the sum. First note that
$\D(Y_0)=0$ as $Y_0$ is a constant.  The value of $\D\big(Y_t \mid
Y_{\anc(t)}\big)$ is computed by considering three separate cases. If
$X_t = X_{\rho(t)}$ (i.e., the player chooses the same action on
rounds $t$ and $\rho(t)$) then the distribution of $Y_t$ conditioned
on $Y_{\anc(t)}$ is $N(Y_{\rho(t)}, \sig^2)$ under both $\Q_0$ and
$\Q_i$, where $N(\mu,\sig^2)$ denotes the normal distribution with
mean $\mu$ and variance $\sig^2$.
If $X_t = i$ and $X_{\rho(t)} \ne i$ then the distribution of
$Y_{t}$ conditioned on $Y_{\anc(t)}$ is $N(Y_{\rho(t)},\sig^2)$
under $\Q_0$ and $N(Y_{\rho(t)}-\eps,\sig^2)$ under $\Q_i$. Finally,
if $X_t \ne i$ and $X_{\rho(t)} = i$ then the distribution of $Y_{t}$
conditioned on $Y_{\anc(t)}$ is $N(Y_{\rho(t)},\sig^2)$ under $\Q_0$
and $N(Y_{\rho(t)}+\eps,\sig^2)$ under $\Q_i$. Overall,
\begin{align} \label{eq:kl2}
	\D\big(Y_t \mid Y_{\anc(t)}\big)
	&= \Q_0\left(X_t = i, X_{\rho(t)} \ne i\right) \cdot \KL \left(N(0,\sig^2) \,\big\|\, N(-\eps,\sig^2) \right) \non\\
		&\phantom{=}+ \Q_0\left(X_t \ne i, X_{\rho(t)} = i\right) \cdot \KL \left(N(0,\sig^2) \,\big\|\, N(\eps,\sig^2) \right) \non\\
	&= \frac{\eps^2}{2\sig^2} \, \Q_0 (A_t) ~ ,
\end{align}
where $A_t = \set{X_t = i, X_{\rho(t)} \ne i \;\vee\; X_t \ne i, X_{\rho(t)} = i}$ is the event that the player switched an odd number of times (and in particular, at least once) from or to action $i$ between rounds $\rho(t)$ and $t$.
Substituting \eqref{eq:kl2} into \eqref{eq:klchain}
gives
\begin{align} \label{eq:klsum}
  \D(Y_{0:T})
	&~=~ \frac{\eps^2}{2\sig^2} \, \sum_{t=1}^T \Q_0(A_t)
	~=~ \frac{\eps^2}{2\sig^2} \, \E_{\Q_0}\left[ \sum_{t=1}^T \ind{A_t} \right] \,.
\end{align}
The event $A_t$ implies that there exists at least one time $s$ of switch from or to action $i$, such that $t \in \cut(s)$.
Therefore, if we let $S_{1:M_i}$ denote
the random sequence of times of such switches (in the action sequence $X_{1:T}$), then
\begin{align*}
	\sum_{t=1}^T \ind{A_t}
	~\le~ \sum_{r=1}^{M_i} \sum_{t \in \cut(S_r)} \ind{A_t}
	~\le~ \sum_{r=1}^{M_i} \abs{\cut(S_r)}
	~\le~ \width(\rho)\,M_i ~.
\end{align*}
Plugging this inequality back into \eqref{eq:klsum} gives
$$
\D(Y_{0:T}) ~\leq~ \frac{\eps^2 \width(\rho)}{2\sig^2} \, \E_{\Q_0}[M_i] ~.
$$
Pinsker's inequality (Lemma 11.6.1 in \cite{cover2006elements}) now implies that
$$
\sup_{A \in \Fcal'} \big(\Q_0(A) - \Q_i(A) \big)
~\le~ \frac{\eps}{2\sig} \sqrt{\width(\rho) \, \E_{\Q_0}[M_i]} ~,
$$
where $\Fcal'$ is the $\sigma$-algebra generated by $Y_{0:T}$. We can
replace $\Fcal'$ with $\Fcal$ above to obtain $\tv{\Q_0}{\Q_i}$ in the left-hand side,
simply because $Z_{1:T}$ is a deterministic function of $Y_{0:T}$ and therefore $\Fcal \subset \Fcal'$.

This proves the first claim of the lemma.
To prove the second bound, we can simply reverse the roles of $\Q_0$ and $\Q_i$ in our arguments above and obtain the same bound over the total variation distance but in terms of the expectation with respect to the distribution $\Q_i$.
\end{proof}

\subsection{Regret Lower Bound}\label{app:lower}

With Lemma~\ref{lem:tv-bound} in hand, we can prove \thmref{thm:regret-lb} and conclude \thmref{thm:main}. 
We begin with a simple corollary of the lemma.

\begin{corollary} \label{cor:tv-bound}
It holds that
$
	\frac{1}{k} \sum_{i=1}^{k} \tv{\Q_0}{\Q_i}
	\le \frac{\eps}{\sig \sqrt{2k}} \cdot \sqrt{\width(\rho) \, \E_{\Q_0}[M]} .
$
\end{corollary}

\begin{proof}
Averaging the inequalities of \lemref{lem:tv-bound} over
$i=1,2,\ldots,k$, using the concavity of the root function and noting
that $\sum_{i=1}^{k} M_i = 2 M$ (as each switch is counted twice in
the sum) yields
$$
\frac{1}{k} \sum_{i=1}^{k} \tv{\Q_0}{\Q_i}
~\le~ \frac{\eps}{2\sig} \cdot \frac{1}{k} \sum_{i=1}^{k} \sqrt{\width(\rho) \, \E_{\Q_0}[M_i]}
~\le~ \frac{\eps}{\sig \sqrt{2k}} \cdot \sqrt{\width(\rho) \, \E_{\Q_0}[M]} ~,
$$
as claimed.
\end{proof}

We now turn to analyzing the player's expected regret.
Using the definitions above, this regret can be written as
$$
R ~=~ \sum_{t=1}^T L_t(X_t) + M - \min_{x\in[k]}\sum_{t=1}^{T} L_t(x)~.
$$
As a tool in our analysis, we also define the hypothetical regret
with respect to the unclipped loss functions $L_{1:T}'$ that the
player would suffer on the \emph{same action sequence} $X_{1:T}$. Namely,
$$
R' = \sum_{t=1}^{T} L_t'(X_t) + M - \min_{x\in[k]}\sum_{t=1}^{T} L'_t(x)~.
$$

The next lemma shows that in expectation, the regret $R$ can be lower
bounded in terms of $R'$.   

\begin{lemma} \label{lem:regtilde}
Assume that $T \ge \max\set{k, 6}$.
Then~
$
	\E[R] \ge \E[R'] - \eps T/6 \,.
$
\end{lemma}

\begin{proof}
We consider the event $B = \{ \forall t : L_t = L_t' \}$, and first show that $\Pr(B) \ge 5/6$.
As the process $W_{1:T}$ has depth $d \le \floor{\log_2{T}}+1 \le 2\log_2{T}$, Lemma~\ref{lem:depth} with $\delta = 1/T \le 1/6$ implies that with probability at least~$5/6$, we have
\begin{align*}
	\abs{W_t}
	\le \sig \sqrt{2d \log\tfrac{T}{\delta}}
	\le \sig \sqrt{8 \log_2{T} \log{T}}
	\le 3\sig \log_2{T}
\end{align*}
for all $t \in [T]$.
Thus, setting $\sig = 1/(9 \log_2{T})$ we obtain that 
$$\Pr\lr{\forall t \in [t] \quad \half + W_t \in \left[ \frac{1}{6}, \frac{5}{6} \right]} \ge \frac{5}{6} \,.
$$ 
For  $T \ge \max\set{k,6}$ we have $\eps < 1/6$ and thus $L_t'(x) \in [0,1]$ for all $x \in [k]$ whenever  $\half + W_t \in [\tfrac{1}{6},\tfrac{5}{6}]$.  
This implies that $\Pr(B) \ge 5/6$.

If $B$ takes place then $R=R'$; 
otherwise,   $M \le R \le R' \le M + \eps T$ so that $R' - R \le \eps T$.
Therefore,
$
	\E[R'] - \E[R]
	= \E[R' - R \mid \neg B]  \cdot \Pr(\neg B)
	\le \eps T /6 ,
$
as required.
\end{proof}

Next, we relate the hypothetical regret $R'$ to the total variation between $\Q_0$ and the $\Q_i$.

\begin{lemma} \label{lem:regret-tv}
The quantity $\E[R']$ is lower bounded in terms of the distributions $\Q_0,\Q_1,\ldots,\Q_k$ as
\begin{align*}
	\E[R']
	\ge \frac{\eps T}{2} - \frac{\eps T}{k} \cdot \sum_{i=1}^{k} \tv{\Q_0}{\Q_i} + \E[M] ~.
\end{align*}
\end{lemma}

\begin{proof}
For $i \in [k]$, let   $N_i$ denote the number of times the player picks action~$i$,
so we can write $R' = \eps \, (T - N_\chi) + M$.
Consequently,  
\begin{align} \label{eq:rPrime}
	\E[R']
	= \frac{1}{k} \sum_{i=1}^{k} \E[\eps \, (T - N_i) + M \mid \chi = i]
	= \eps T - \frac{\eps}{k} \sum_{i=1}^{k} \E_{\Q_i}[N_i] + \E[M] ~.
\end{align}
On the other hand,  for all $i \in [k]$ and $t \in [T]$, the event $\set{X_t = i}$ is in the $\sigma$-field $\mathcal{F}$, so
$
	\Q_i(X_t = i) - \Q_0(X_t = i)
	\le \tv{\Q_0}{\Q_i} ~.
$
Summing over $t=1,\ldots,T$ yields
$
	\E_{\Q_i}[N_i] - \E_{\Q_0}[N_i]
	\le T \cdot \tv{\Q_0}{\Q_i} ,
$
whence
\begin{align*}
	\sum_{i=1}^{k} \E_{\Q_i}[N_i]
	\le T \cdot \sum_{i=1}^{k} \tv{\Q_0}{\Q_i} + \sum_{i=1}^{k} \E_{\Q_0}[N_i]
	= T \cdot \sum_{i=1}^{k} \tv{\Q_0}{\Q_i} + T ~.
\end{align*}
Plugging this into \eqref{eq:rPrime} and using $k \ge 2$ gives
\begin{align*}
	\E[R']
	&\ge \eps T - \frac{\eps T}{k} \cdot \sum_{i=1}^{k} \tv{\Q_0}{\Q_i} - \frac{\eps T}{k} + \E[M] \\
	&\ge \frac{\eps T}{2} - \frac{\eps T}{k} \cdot \sum_{i=1}^{k} \tv{\Q_0}{\Q_i} + \E[M] ~,
\end{align*}
as claimed.
\end{proof}

We are now ready to prove~\thmref{thm:regret-lb}. \\

\begin{proof}[Proof of~\thmref{thm:regret-lb}]
We first prove the theorem for deterministic players that make no more than $\eps T$ switches on \emph{any} sequence of loss functions, and relax this assumption towards the end of the proof.
For algorithms with this property, we have $\Q_0(M > \eps T) = \Q_i(M > \eps T) = 0$ for all $i \in [k]$.
Since $\set{M \ge m} \in \Fcal$, this implies
$$
	\E_{\Q_0}[M] - \E_{\Q_i}[M]
	= \sum_{m=1}^{\floor{\eps T}} \lr{\Q_0(M \ge m) - \Q_i(M \ge m)}
	\le \eps T \cdot \tv{Q_0}{Q_i}
$$
for all $i \in [k]$, that gives
\begin{align*}
	\E_{\Q_0}[M] - \E[M]
	= \frac{1}{k} \sum_{i=1}^{k}\lr{\E_{\Q_0}[M] - \E_{\Q_i}[M]}
	\le \frac{\eps T}{k} \sum_{i=1}^{k}\tv{Q_0}{Q_i} ~.
\end{align*}
Combining this with the results of \lemref{lem:regtilde} and \lemref{lem:regret-tv}, we obtain
\begin{align*}
	\E[R]
	\ge \frac{\eps T}{3} - \frac{2\eps T}{k} \sum_{i=1}^{k} \tv{\Q_0}{\Q_i} + \E_{\Q_0}[M] ~.
\end{align*}

On the other hand, recall~\lemref{lem:mrw} that states that the width of the MRW process is bounded by $w(\rho) \le \floor{\log_2{T}} + 1 \le 2 \log_2{T}$.
Corollary~\ref{cor:tv-bound} together with this bound gives
\begin{align*}
	\frac{1}{k} \sum_{i=1}^{k} \tv{\Q_0}{\Q_i}
	\le \frac{\eps}{\sig \sqrt{k}} \cdot \sqrt{\E_{\Q_0}[M] \, \log_2{T}} ~.
\end{align*}
Plugging this into the previous inequality and using the notation $m = \sqrt{\E_{\Q_0}[M]}$ results with the lower bound
\begin{align*}
	\E[R]
	\ge \frac{\eps T}{3} + m \lr{m - \frac{2\eps^2}{\sig \sqrt{k}} \; T \sqrt{\log_2{T}}} ~.
\end{align*}
The right hand side, which is minimized at $m = (\eps^2/ \sig \sqrt{k}) \, T \sqrt{\log_2{T}}$, can be further lower bounded by
$
	\eps T/3 - (\eps^4 / \sig^2 k) \, T^2 \log_2{T} \,.
$
Using our choice of $\sig = 1/(9\log_2{T})$ and $\eps = k^{1/3} T^{-1/3} / (9\log_2{T})$ gives
\begin{align} \label{eq:lbtemp}
	\E[R]
	\ge \lr{\frac{1}{27} - \frac{1}{81}} \cdot \frac{k^{1/3} T^{2/3}}{\log_2{T}}
	\ge \frac{k^{1/3} T^{2/3}}{50\log_2{T}} ~.
\end{align}

This proves the theorem for algorithms with the assumed property.  In
order to relax this assumption, note that we can turn any player
algorithm to an algorithm that makes at most $\eps T$ switches, simply
by halting the algorithm once it makes $\floor{\eps T}$ switches and
repeating its last action on the remaining rounds.  The regret $R^*$
of the modified algorithm equals $R$ unless $M >\eps T$ and in the
latter case $R^* \le R+\eps T \le 2R$, so $\E[R^*] \le 2 \, \E[R]$. Since
$\E[R^*]$ is lower bounded by the right-hand side of
\eqref{eq:lbtemp}, this implies the claimed lower bound on the expected regret of
any deterministic player.
\end{proof}

Finally, we can prove \thmref{thm:main}.

\begin{proof}[Proof of~\thmref{thm:main}]
Recall that any randomized algorithm is equivalent to an a-priori
random choice of a deterministic algorithm, for which the statement
of~\thmref{thm:regret-lb} applies.  Hence, since the adversary is
oblivious to the player's actions, the statement
of~\thmref{thm:regret-lb} for a randomized player~(where the
expectation is now taken with respect to both the functions $L_{1:T}$
and the player's random bits) follows by taking the expectation over
its internal randomization.  The fact the expectation of the regret
with respect to the randomization in $L_{1:T}$ is lower bounded by the
stated quantity implies that there exists some realization
$\ell_{1:T}$ of the variables $L_{1:T}$ for which the regret is lower
bounded by the same quantity.  This gives the result
of~\thmref{thm:main}.
\end{proof}

\section{Extensions and Implications}\label{sec:extensions}

In this section we present few extentions of our results and discuss several implications.

\subsection{Binary losses}

In our construction of a randomized adversary, described in
\secref{sec:construction}, the loss values $L_t(x)$ are all real
numbers in the interval $[0,1]$.  One might wonder whether a similar
construction exists where each of the loss values is constrained to be
either $0$ or $1$. A simple adaptation of our construction shows that
this is indeed the case. To see this, simply set the loss of action
$x$ at time $t$ to be the outcome of a biased coin toss with bias
$L_t(x)$. In this sequence of binary loss functions, action $\chi$ is
consistently better \emph{in expectation} by an $\eps$ gap, which is
sufficient in our analysis. Our arguments regarding the player's
inability to identify the best action still apply since the feedback
he observes is only further obscured by additional random noise.


\subsection{Arbitrary Switching Cost}

Assume that each switch incurs a cost of $c$ to the player, instead of
a unit cost as before.  Repeating the proof of~\thmref{thm:regret-lb},
we are able to get an $\t{\Omega}(c^{1/3} k^{1/3} T^{2/3})$ lower bound,
which is tight with respect to $T$, $k$ and $c$
(up to poly-log factors) in light of the upper bound of \citet{Arora:12}.

\begin{theorem} \label{thm:switch-cost}
Let the cost of switch be $c > 0$ and assume that $T > c \cdot \max\set{k,6}$.
For any randomized player strategy that relies on bandit feedback,
there exists a sequence of loss functions $\ell_{1:T}$ (where
$\ell_t:[k] \mapsto [0,1]$) that incurs a regret
of $R = \widetilde \Omega(c^{1/3} k^{1/3} T^{2/3})$.
\end{theorem}

\begin{proof}
Redefine the gap between the actions in the construction of the functions $L_{1:T}$ to $\eps = (c k)^{1/3} T^{-1/3}/(9 \log_2{T})$.
Using the same
notation as in the proof of~\thmref{thm:regret-lb}, we can show that
\begin{align*}
	\E[R]
	\ge \frac{\eps T}{3} + m \lr{c m - \frac{2\eps^2}{\sig \sqrt{k}} \; T \sqrt{\log_2{T}}} ~.
\end{align*}
The right-hand side is minimized at $m = (\eps^2/ c \sig \sqrt{k}) \, T \sqrt{\log_2{T}}$ and is lower bounded by
$\eps T/3 - (\eps^4 / \sig^2 c k) \, T^2 \log_2{T}$.
Setting $\eps = (c k)^{1/3} T^{-1/3}/(9 \log_2{T})$ and using our choice of $\sig = 1/(9 \log_2{T})$ gives the lower bound
\begin{align*}
	\E[R]
	\ge \frac{c^{1/3} k^{1/3} T^{2/3}}{50 \log_2{T}} \,.
\end{align*}
Proceeding as in the proofs of~\thmref{thm:regret-lb} and~\thmref{thm:main}, we establish the existence of the required sequence of loss functions $\ell_{1:T}$.
\end{proof}

\subsection{Tradeoff between Loss and Switches}

As a corollary of~\thmref{thm:switch-cost}, we can quantify the
tradeoff between the loss accumulate by a multi-armed bandit algorithm
and the number of switches it performs.
For simplicity, we treat the number of actions $k$ as a constant and state the result only in terms of $T$.

\begin{theorem}\label{thm:lossVsSwitch}
Let $\Acal$ be a multi-armed bandit algorithm that guarantees an
expected regret (without switching costs) of $\t{O}(T^\alpha)$
then there exists a sequence of loss functions that forces $\Acal$ to make
$\t\Omega(T^{2(1-\alpha)})$ switches.
\end{theorem}

In particular, the popular EXP3 algorithm \citep{Auer:02} guarantees a regret of
$O(\sqrt{T})$ without switching costs. In this case,
\thmref{thm:lossVsSwitch} implies that EXP3 can be forced to make
$\t\Omega(T)$ switches.

\begin{proof}[Proof of \thmref{thm:lossVsSwitch}]
Assume the contrary, i.e.~that $\Acal$ can guarantee a regret of
$\t{O}(T^\alpha)$ (without switching costs) with $\t O(T^{\beta})$
switches over any sequence of $T$ loss functions, with $\alpha + \beta/2 < 1$.
In this case, we can pick a real number $\gamma$ such that $\alpha < \gamma < 1-\beta/2$.
Consider the performance of
this algorithm in a setting where the cost of a switch is~$c =
T^{3\gamma-2}$.  Clearly, the expected regret (including switching
costs) of the algorithm in this setting is upper bounded by
$$
\t{O}(T^\alpha + T^{3\gamma-2} \cdot T^{\beta}) ~=~ \t{o}(T^{\gamma})~~,
$$
over any sequence of loss functions, as~$\alpha < \gamma$ and~$\beta <
2-2\gamma$.  This contradicts~\thmref{thm:switch-cost}, which
guarantees the existence of a loss sequence that incurs a regret (including switching costs) of
$\t{\Omega}(T^{(3\gamma-2)/3} \cdot T^{2/3}) = \t{\Omega}(T^\gamma)$.
\end{proof}

 \subsection{Lower Bound for Online Adversarial Markov Decision Processes}

 The multi-armed bandit problem with switching costs is a special case
 of the online adversarial deterministic Markov decision process
 (ADMDP) with bandit feedback (see \citet{dekel2013better} for a formal
 description of this setting). The important aspect of the ADMDP
 setting is that the player has a state, and that his loss on each
 round depends both on his action and on his current state. Moreover,
 the player's action on round $t$ determines his state on round
 $t+1$. The $k$-armed bandit problem with switching costs can be
 described as a $k$-state ADMDP, where each state represents the
 player's previous action. The player incurs the loss associated with
 the action he chooses and pays an additional cost whenever he changes his state.

 As a result, our lower bound applies to the class of ADMDP problems.
 \citet{dekel2013better} proves a matching upper bound, which implies
 that the (undiscounted) minimax regret of the ADMDP problem is
 $\t\Theta(T^{2/3})$. The ADMDP setting belongs to the more general
 class of adversarial MDPs with bandit
 feedback~\citep{YuMaSh09,NeuGySzAn10}, where the state transitions are
 allowed to be stochastic. This implies a $\t\Omega(T^{2/3})$ lower
 bound on the (undiscounted) minimax regret of the general setting.

\section{Summary}
In this paper, we proved that the $T$-round $k$-action multi-armed
bandit problem with switching costs has a minimax regret of
$\widetilde{\Theta}(k^{1/3} T^{2/3})$, and is therefore strictly
harder than the corresponding experts problem (with full feedback).
To the best of our knowledge, this is the first example of a setting in which
learning with bandit feedback is significantly harder than learning with full-information feedback (in terms of the dependence on $T$).
Our analysis shows that the difficulty of this problem stems from the
player's need to \emph{pay for exploring} the quality of the different actions.
Since this problem is a special case of online learning with bandit
feedback against a bounded-memory adaptive adversary, we conclude that the
minimax regret of the general setting is also
$\widetilde\Omega(T^{2/3})$, which matches the upper bounds of
\citet{Arora:12}. We also showed how our construction resolves several
other open problems in online learning. Moreover, we believe that the
multi-scale random walk, defined in \secref{sec:mmrw}, will prove to
be a useful tool in other settings.

\bibliographystyle{plainnat}
\bibliography{bib}

\end{document}

%% file: macro.tex
\renewcommand{\Pr}{\mathbb{P}}

\newcommand{\E}{\mathbb{E}}

\newcommand{\half}{\frac{1}{2}}

\newcommand{\floor}[1]{\lfloor#1\rfloor}

\newcommand{\Acal}{\mathcal{A}}

\newcommand{\Fcal}{\mathcal{F}}

\newcommand{\secref}[1]{Sec.~\ref{#1}}
\newcommand{\figref}[1]{Fig.~\ref{#1}}
\renewcommand{\eqref}[1]{Eq.~(\ref{#1})}
\newcommand{\lemref}[1]{Lemma~\ref{#1}}
\newcommand{\thmref}[1]{Theorem~\ref{#1}}

\newcommand{\ignore}[1]{}

\newcommand{\abs}[1]{\left|#1\right|}
\newcommand{\lr}[1]{\left(#1\right)}
\newcommand{\set}[1]{\left\{#1\right\}}
\renewcommand{\t}[1]{\widetilde{#1}}
\newcommand{\KL}{d_\mathrm{KL}}
\newcommand{\ind}[1]{1\!\!1_{#1}}
\newcommand{\tv}[2]{d^\Fcal_{\mathrm{TV}}(#1,#2)}
\newcommand{\non}{\nonumber}

\newcommand{\eps}{\epsilon}

\newcommand{\Q}{\mathcal{Q}}
\newcommand{\D}{\Delta}

\newcommand{\sig}{\sigma}

\newcommand{\cut}{\mathrm{cut}}
\newcommand{\anc}{\rho^*}
\newcommand{\depth}{d}
\newcommand{\width}{w}
\newcommand{\clip}{\mathrm{clip}}